\documentclass{article}

\usepackage{arxiv}

\usepackage[utf8]{inputenc} 
\usepackage[T1]{fontenc}    
\usepackage{hyperref}       
\usepackage{url}            
\usepackage{booktabs}       
\usepackage{amsfonts}       
\usepackage{nicefrac}       
\usepackage{microtype}      
\usepackage{lipsum}		
\usepackage{graphicx}
\usepackage{natbib}
\usepackage{doi}
\usepackage{algorithm}
\usepackage{algorithmic}
\usepackage{subfigure}

\usepackage{hyperref}

\usepackage{amsmath}
\usepackage{amssymb}
\usepackage{mathtools}
\usepackage{amsthm}

\usepackage[capitalize,noabbrev]{cleveref}

\theoremstyle{plain}
\newtheorem{theorem}{Theorem}[section]

\theoremstyle{definition}

\theoremstyle{remark}

\usepackage[textsize=tiny]{todonotes}

\usepackage{amsmath}
\usepackage{amsthm}
\usepackage[singlelinecheck=false,justification=raggedright]{caption}

\title{Preference as Reward, \\ Maximum Preference Optimization with Importance Sampling}

\author{
\hspace{1mm}Zaifan Jiang\thanks{Corresponding author} \\
\texttt{jiangzaifan@shizhuang-inc.com} \\
\And
\hspace{1mm}Xing Huang \\
\texttt{huangxing1231@shizhuang-inc.com} \\
\And
\hspace{1mm}Chao Wei \\
\texttt{weichao@shizhuang-inc.com} \\	
}

\renewcommand{\headeright}{Preprint. Work in Progress.}
\renewcommand{\undertitle}{Preprint. Work in Progress.}
\renewcommand{\shorttitle}{Maximum Preference Optimization}

\hypersetup{
pdftitle={A template for the arxiv style},
pdfsubject={q-bio.NC, q-bio.QM},
pdfauthor={David S.~Hippocampus, Elias D.~Striatum},
pdfkeywords={First keyword, Second keyword, More},
}

\begin{document}
\maketitle

\begin{abstract}
Preference learning is a key technology for aligning language models with human values.
Reinforcement Learning from Human Feedback (RLHF) is a model-based algorithm to optimize preference learning,
which first fits a reward model for preference scores and then optimizes the generating policy with an on-policy PPO algorithm to maximize the reward.
The processing of RLHF is complex, time-consuming, and unstable.
The Direct Preference Optimization (DPO) algorithm uses an off-policy algorithm to directly optimize the generating policy and eliminates the need for a reward model. DPO is more data-efficient and stable. However, DPO has a drawback of overfitting to the preference data and ignoring the KL-regularization term when the preference is deterministic.
Identity mapping Preference Optimization(IPO) uses a root-finding MSE loss to incorporate KL-regularization. However, both DPO and IPO fail to properly address the KL-regularization term because the support of the preference distribution is not equal to the reference distribution.
In this paper, we propose a simple and intuitive off-policy preference optimization algorithm from an importance sampling view, which we call Maximum Preference Optimization (MPO). MPO incorporates the off-policy KL-regularization term, making regularization truly effective. MPO achieves the best of both worlds by combining the objectives of RLHF and IPO while being an off-policy algorithm. Furthermore, MPO eliminates the need for a reward model and reference policy, simplifying the learning process and reducing memory usage.
\end{abstract}


\section{Introduction}
Large language models (LLMs) \cite{brown2020language} \cite{chowdhery2023palm} \cite{bubeck2023sparks} \cite{radford2019language} with massive scale parameters trained on a large amount of data using pretrain,
supervised fine-tune (SFT) \cite{wei2021finetuned,narayanan2021efficient,sanh2021multitask}, and instruction fine-tune (IFT) \cite{chung2022scaling,mishra2021cross,thoppilan2022lamda} algorithms 
have lead to surprising capabilities like few-shot in context learning.
The training dataset comes from variety of areas and has different qualities,
and the training algorithms (pretrain, SFT, IFT) all rely on maximum likelihood estimation (MLE) which learn to match the distribution 
of the training dataset.
LLMs trained on these data using the MLE algorithm generate contents with a quality gap compared to human judgement or values.

Preference learning \cite{ziegler2019fine} \cite{bai2022training} \cite{christiano2017deep} \cite{stiennon2020learning} algorithms 
significantly improve the generating quality to align with human values.
After pairs of generations under the same context are collected, the pairwise human preference
is labeled to indicate which generation is better. 
Then, a preference learning algorithm is used to optimize the generating policy to align with human values.
While relative human judgements are easier to collect than expert-labeled data, subsequent works use
the preference learning algorithm to improve proficiency\cite{kreutzer2018reliability,stiennon2020learning,ziegler2019fine} and instruction following\cite{ouyang2022training,ramamurthy2022reinforcement}.

Reinforcement learning from human (or AI) feedback (RLHF/RLAIF)\cite{ouyang2022training,bai2022constitutional} use
reward model-based reinforcement learning algorithm to learn the optimal policy.
It first learns a reward model from the preference data,
then uses an on-policy PPO \cite{schulman2017proximal} algorithm to maximize the learned reward.
The reward is learned with the Bradley-Terry model \cite{bradley1952rank,bong2022generalized},
which assumes the preference score can be approximated by substituted with point-wise rewards.
This assumption may lead to an approximation error when preference is deterministic.
The PPO algorithm is used on data sampled from the generating policy,
which may have a different support or distribution drift from preference data.
The learned reward model inference on the out-of-distribution data may reduce the accuracy.
The process of RLHF involves training a reward model and searching for a policy using on-policy PPO algorithm which is complex, time-consuming, and unstable.

Direct preference optimization (DPO)\cite{rafailov2023direct} combines an off-policy algorithm and the Bradley-Terry model 
to directly learns the generating policy from preference data.
The off-policy algorithm is based on KL-regularization reward maximization from the off-RL
\cite{levine2020offline,schulman2017equivalence,nachum2017bridging} community,
which is data efficient, stable and eliminates the need for a reward model.
When preferences are deterministic, which occurs in most cases, the reward of the Bradley-Terry model is undefined,
which leads to ignoring the KL-regularization term and over-fitting the preference dataset.

Identity mapping preference optimization (IPO)\cite{azar2023general} 
is another off-policy algorithm that incorporates KL-regularization to learn the generating policy from preference data.
It uses root finding mean square error(MSE) loss 
to maximize the probability of preferences while considering KL-regularization.
However, both DPO and IPO fail to properly account for the KL-regularization term due to the mismatch between the support of the preference data distribution and the reference policy distribution.

In this paper, 
we design a simple and intuitive off-policy maximum preference optimization (MPO) algorithm from an importance
sampling view and incorporate an off-policy KL-regularization term to make the KL-regularization truly
effective.
Our key contributions of this paper can be summarized as follows:
\begin{itemize}
    \item formalize preference learning as a preference/reward maximization problem, and design a simple and intuitive off-policy algorithm from importance sampling view
    \item figure out KL-regularization fails when optimized on preference data, and design an off-policy sample loss to make KL-regularation truly effective
	\item eliminate the reward substitution assumption and out-of-distribution generalization assumption
	\item eliminate the needs for both reward model and reference policy to reduce memory usage
\end{itemize}


\begin{figure*}
	\centering
\includegraphics[width=.9\textwidth]{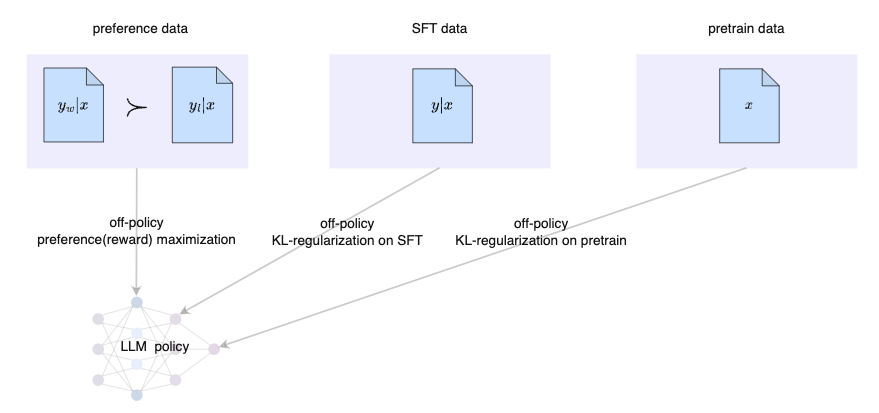}
\caption{Maximum Preference Optimization (MPO) direct optimize preference maximization on preference data using off-policy algorithm,
and use offline SFT, pretrain data to make KL-regularation truly effective, which also eliminate the needs for both reward model and reference policy.}
\end{figure*}

\section{Preliminaries}
The main pipeline of preference learning usually consists of three phases: 1) pretraining and supervised fine-tuning (SFT), where SFT is not a must; 2) preference data collection;
3) reinforcement-learning optimization.

\paragraph{Pretraining and SFT phase}
Preference learning typically starts with a pretrained LLM or a LLM fine-tuned on high quality data using maximum likelihood estimation.
We define the final policy after this phase as $\pi_{\text{ref}}$,
and the data to train $\pi_{\text{ref}}$ as $\mathcal{D}_\text{ref}$. So,
\begin{equation} \label{ref policy}
  \pi_{\text{ref}}\approx\arg\max_{\pi}\mathbb{E}_{x,y\sim\mathcal{D}_\text{ref}}\log \pi(y|x)
\end{equation}

\paragraph{Preference data collection phase}
After pretraining and SFT phase, $\pi_{\text{ref}}$ is prompted by context $x$, 
and generates two responses $y_w, y_l \sim \pi_\textnormal{ref}(\cdot|x)$.
Then $x,y_w,y_l$ is labeled by humans to judge which response is preferred. Denote $y_w\succ y_l | x$ if $y_w$ is preferred,
and $y_l\succ y_w|x$ if $y_l$ is preferred. We define a new symbol $I=\mathbb{I}[y_w\succ y_l|x]$, and all $\langle x,y_w,y_l,I \rangle$ constitute the preference dataset $\mathcal{D}^p$:
\begin{equation}
  \langle x,y_w,y_l,I\rangle\sim\mathcal{D}^p.
\end{equation}
We also define $\rho$ as the context distribution of $x$ and $\mu$ as the preference pair distribution given context $x$ from preference data distribution
\begin{equation}\label{rho}
  x\sim\rho,
\end{equation}
\begin{equation}\label{mu}
  \langle y_w,y_l \rangle\sim\mu(\cdot|x).
\end{equation}
\begin{equation}\label{pstar}
  I \sim \text{Bernoulli}(p^*(y_w\succ y_l|x)),
\end{equation}
where $p^*(y_w\succ y_l|x)$ denotes the preference probability of $y_w \succ y_l$ given context $x$.

\paragraph{Reinforcement-learning optimization phase}
In the final phase, the prevailing approach involves using a reinforcement learning algorithm to learn an explicit or implicit reward from preference data. Subsequently, an on-policy or off-policy policy gradient algorithm is used to maximize this reward.

Recently, some methods have taken a different approach by deriving the optimal target distribution through reward maximization under KL-regularization. They design a loss function using empirical data, with the optimal target distribution serving as the solution to this loss function.

In this paper, we introduce a novel pair-wise preference reward $r^p(y_w\succ y_l|x)=p^*(y_w\succ y_l|x)$, 
and design an algorithm to directly optimizes the preference(reward) maximization objective.

\paragraph{Local distribution introduced by a preference pair $\langle x, y_w, y_l\rangle$}
Given policy $\pi$, there is a global distribution $\pi(\cdot | x)$. 
Conditioned by $y_w,y_l$, we introduce a local distribution $\pi^p$:
\begin{equation}\label{pair:policy}
    \begin{aligned}
    \pi^p(y_w\succ y_l|x,y_w,y_l)
    &=\frac{\pi(y_w|x)}{\pi(y_w|x)+\pi(y_l|x)}\\
    &=\sigma(\log \pi(y_w|x) -\log\pi(y_l|x)).
    \end{aligned}
\end{equation}
In order to simplify notation, we also use
$$
\pi^p(y_w|x)=\pi^p(y_w\succ y_l|x,y_w,y_l).
$$

\section{Background}
\subsection{Reinforcement Learning from Human Feedback (RLHF)}

The RLHF uses reward model-based reinforcement learning algorithm to learn preferences from human feedback.
It consists two main steps:
1) reward estimation from preference data and 2) reward maximization using the PPO algorithm.

\paragraph{Reward estimation from preference data}
In previous research, the point-wise reward is learned with the Bradley-Terry model. Given context x,
we define $r^*(x,y)$ as the reward of generating $y$. 
The Bradley-Terry model assumes the preference probability $p^*(y_w\succ y_l|x)$ to be:
\begin{equation}
    \begin{aligned}
  p^*(y_w\succ y_l|x) &= \frac{\exp(r^*(x,y_w)}{\exp(r^*(x,y_w)) + \exp(r^*(x,y_l))}\\
  &=\sigma(r^*(x,y_w)-r^*(x,y_l))
    \end{aligned}
  \label{BT model}
\end{equation}
where $\sigma(\cdot)$ is the sigmoid function.

RLHF uses \ref{BT model} to model the point-wise reward, and optimize log loss to estimate the reward.
The estimated reward is defined as parameterized $r_\phi$, and the loss function is defined as:
\begin{equation}
    \begin{aligned}
  \mathcal{L}(r_\phi)=&-\underset{\langle x,y_w,y_l,I\rangle\sim \mathcal{D}^p}{\mathbb{E}}[
    I\cdot\sigma(r_\phi(x,y_w) - r_\phi(x,y_l)) \\
    &+ (1 - I)\cdot\sigma(r_\phi(x,y_l)-r_\phi(x,y_w))].
    \end{aligned}
    \label{reward loss}
\end{equation}

The loss \ref{reward loss} is used for maximum likelihood estimation,
and the estimated reward $r_\phi$ is used to
approximate the probability $p^*(y_w\succ y_l|x)$ from the preference data $\mathcal{D}^p$.

\paragraph{Reward maximization using PPO algorithm} 
The reward-maximization or KL-regularized reward-maximization objective is used for reinforcement learning based policy optimization:
\begin{equation}
  \max_{\pi_\theta }\underset{{\substack{x\sim\rho\\y\sim \pi_\theta(\cdot|x)}}}{\mathbb{E}}[r_\phi(x,y)]-\beta \mathbb{D}_{\text{KL}}[\pi_\theta||\pi_{\text{ref}}]
  \label{reg rl}
\end{equation}
where $\mathbb{D}_{\text{KL}}$ is the $\text{KL}$-divergence and $\beta$ is the regularization weight. 
This objective is optimized by on-policy REINFORCE \cite{mnih2016asynchronous} or PPO\cite{schulman2017proximal} algorithm.

The second phase of RLHF involves optimizing objective \ref{reg rl} using $r_\phi$ learned from \ref{reward loss}.
PPO is an on-policy algorithm which continues to collect sample from the current policy $\pi_\theta$ and collects reward from the estimated reward model.
It then uses the collected data to estimate the gradient of \ref{reg rl}, and then updates the current policy.
Because $\pi_\theta$ is different from $\pi_{\text{ref}}$ defined as \ref{ref policy},
samples generated by $\pi_\theta$ has a different distribution of $\mathcal{D}^p$. So, RLHF assumes
$r_\phi$ can generalize to out-of-distribution samples generated by $\pi_\theta$.


RLHF also mixes a pretraining gradient into the PPO objective, in order to fix the performance regression
on public NLP datasets. And RLHF call the final objective 'PPO-ptx'. Define $\mathcal{D}_\text{pretrain}$ as the pretraining dataset, then the combined objective is defined as:
\begin{equation}\label{ppo-ptx}
    \begin{aligned}
	\max_{\pi_\theta}
    &\underset{\substack{x\sim\rho\\y \sim \pi_\theta(\cdot|x)}}{\mathbb{E}}\bigl[
		r_\phi(x,y)
	]
	-\beta \log(\frac{\pi_\theta(y|x)}{\pi_\text{ref}(y|x)})\bigr]\\
	&+\gamma \underset{\mathcal{D}_\text{pretrain}}{\mathbb{E}}[\pi_\theta(x)].
    \end{aligned}
\end{equation}

\subsection{Direct Preference Optimization (DPO)}
An alternative RL algorithm to preference learning is direct preference optimization (DPO), 
which eliminates the training of a reward model.
Following prior work \cite{rafailov2023direct,nachum2017bridging,schulman2017equivalence}, it is straightforward to show that the optimal solution $\pi_r$ of \ref{reg rl} for reward $r(x,y)$ takes the form:
\begin{equation}
  \pi_r(y|x)\propto \pi_{\text{ref}}(y|x)\exp(\frac{1}{\beta}r(x,y)).
  \label{reg policy}
\end{equation}

Given an optimal policy $\pi^*$ of \ref{reg rl}, DPO derives an implicit reward from Eq. \ref{reg policy}:
\begin{equation}
  r^*(x,y)=\beta\log\frac{\pi^*(y|x)}{\pi_{\text{ref}}(y|x)}+\beta\log Z(x),
  \label{reg reward}
\end{equation}
where $Z(x)=\sum_y\pi_{\text{ref}}(y|x)\exp(\frac{1}{\beta}r^i(x,y))$ is the partition function. 
Substituting the re-parameterized reward from \ref{reg reward} into the Bradley-Terry model \ref{BT model}:
\begin{equation} \label{pair score}
\begin{split}
p^*(y_w\succ y_l|x)&=\sigma(r^*(x,y_w)-r^*(x,y_l))\\
&=\frac{1}{1 + \exp\big(
r^*(x,y_w) - r^*(x,y_l)
\big)}
\end{split}.
\end{equation}

Substituting the probability \ref{pair score} to the log loss \ref{reward loss},
DPO formulates a maximum likelihood objective for the parameterized policy $\pi_\theta$ with the empirical preference
dataset $\mathcal{D}^p$:
\begin{equation}\label{loss:dpo}
  \begin{aligned}
  \mathcal{L}_{\text{DPO}}(\pi_\theta;\pi_{\text{ref}})=&
  -\underset{\langle x,y_w,y_l,I\rangle\sim \mathcal{D}^p}{\mathbb{E}}\bigl[\\
    I&\log \sigma(
    r_\theta(x,y_w) - r_\theta(x,y_l)
)\\
+ (1-I)&
\log \sigma(
    r_\theta(x,y_l) - r_\theta(x,y_w)
)
  \bigr],
  \end{aligned}
\end{equation}
where $r_\theta(x,y)=\beta\log\frac{\pi_\theta(y|x)}{\pi_\text{ref}(y|x)}$ is the reward implicitly defined by $\pi_\theta$
and $\pi_\text{ref}$.

Although this pair-wise loss eliminates the need to calculate the partition $Z(x)$, 
it also makes the optimal solution $\pi^*_\theta$ undefined when there are not enough constraints.
For example, if weight $\pi_\theta(y_w|x)$ and $\pi_\theta(y_l|x)$ with the same multiplier $M$, the logits of the sigmoid function $\sigma$ 
will remain the same:
\begin{equation}\label{non-unique}
    \begin{aligned}
&\beta\log\frac{\pi_\theta(y_w|x) * M}{
\pi_{\text{ref}}(y_w|x)
}
-
\beta\log\frac{\pi_\theta(y_l|x) * M}{
\pi_{\text{ref}}(y_l|x)
} \\
&= \beta\log\frac{\pi_\theta(y_w|x)}{
\pi_{\text{ref}}(y_w|x)
}
-
\beta\log\frac{\pi_\theta(y_l|x)}{
\pi_{\text{ref}}(y_l|x)
}.
\end{aligned}
\end{equation}
This makes the final learned policy $\pi_\theta$ suboptimal, and also fails the KL-regularation term.

\subsection{$\Psi$-PO with identity mapping (IPO)}

IPO defines a new objective called $\Psi$-preference optimization objective ($\Psi$PO):
\begin{equation}
	\max\limits_{\pi_\theta}\underset{\substack{x\sim\rho \\ y_w\sim \pi_\theta(\cdot|x) \\ y_l\sim \mu(\cdot|x)}}{\mathbb{E}}
	[\Psi(p^*(y_w\succ y_l|x))] - \tau \mathbb{D}_{\text{KL}}[\pi_\theta||\pi_{\text{ref}}],
\end{equation}
where $\Psi$ is a general non-decreasing function $\Psi: [0, 1]\rightarrow \mathbb{R}$.

Take $\Psi$ to be the identity mapping,
IPO derives an off-policy loss on empirical dataset:
\begin{equation}\label{loss:ipo}
    \begin{aligned}
	\min_{\pi_\theta}\underset{\langle x,y_w,y_l, I\rangle \sim\mathcal{D}^p}{\mathbb{E}}
	\bigl[
	&I h_{\pi_\theta}(x, y_w, y_l) - (1 - I)h_{\pi_\theta}(x, y_l, y_w) \\
    &- \frac{\tau^{-1}}{2}
		\bigr]^2,
    \end{aligned}
\end{equation}
where $h_{\pi_\theta}(x,y_w,y_l)$ is defined as:
\begin{equation}\label{hpi}
h_{\pi_\theta}(x,y_w,y_l)=\log\frac{\pi_\theta(y_w|x)}{\pi_\text{ref}(y_w|x)}-\log\frac{\pi_\theta(y_l|x)}{
    \pi_\text{ref}(y_l|x)
}.
\end{equation}

IPO claims when preferences are deterministic or near deterministic, DPO will lead over-fitting to the preference 
dataset at the expense of ignoring the KL-regularation term. And IPO's loss will always regularizes $\pi_\theta$ towards
$\pi_{\text{ref}}$ by controlling the gap between the log-likelihood ratios $\log\frac{\pi_\theta(y_w|x)}{\pi_\theta(y_l|x)}$
and $\log\frac{\pi_{\text{ref}}(y_w|x)}{\pi_{\text{ref}}(y_l|x)}$.

Similar to DPO, the IPO loss controls the ratio of $\frac{\pi_\theta(y_w|x)}{\pi_\theta(y_l|x)}$ not too far
from $\frac{\pi_\text{ref}(y_w|x)}{\pi_\text{ref}(y_l|x)}$, but doesn't control $\pi_\theta(y|x)$ not too far from $\pi_\text{ref}(y|x)$.
When there are not enough constraints, which is almost always the case, the optimal policy is not unique, so the KL-regularation term also fails.

\section{Method}
In this work, we combine the preference maximization term of IPO's loss and the regularization term of RLHF's loss.
Unlike IPO, which derives the off-policy from preference maximization under KL-regularization, we formulate 
preference maximization as a reward maximization problem in the reinforcement learning setting, and derive an 
off-policy objective from an importance sampling view and without the help of 
KL-regularation. Then we combine the off-policy reward maximization objective with forward KL-regularization term, which makes the KL-regularization truly effective. We call the algorithm Maximum 
Preference Optimization (MPO).
The final objective of MPO bears resemblance to RLHF's objective, and likes IPO and DPO, MPO is off-policy.

\subsection{Preference(reward) Maximization with Importance Sampling}
We define preference as reward and formalize preference maximization as a reward maximization problem in 
the reinforcement learning setting.
We define $x,y_w, y_l$ as state, and $\mathcal{A}_{x,y_w,y_l}=\{y_w\succ y_l, y_l \succ y_w\}$ as the action set.
To simplify notation, let $\mathcal{A}=\{y_w,y_l\}$, where $y_w$ represents $y_w\succ y_l$, and
$y_l$ represents $y_l\succ y_w$.
For an action $\mathfrak{a}\in\mathcal{A}$,
define the reward of action $\mathfrak{a}$ as $r^p(\mathfrak{a}|x,y_w,y_l)$, which is the preference probability.
By simplifying $r^p(\mathfrak{a}|x,y_w,y_l)$ as $r^p(\mathfrak{a}|x)$, we get:
\begin{equation}\label{preference:reward}
	\begin{split}
	r^p(\mathfrak{a}|x)&=r^p(\mathfrak{a} | x, y_w, y_l)\\
	&=\mathbb{E}[\mathbb{I}\{\mathfrak{a}\}|x]\\
	&=p^*(\mathfrak{a}|x).
	\end{split}
\end{equation}

Given a sample $x, y_w, y_l, I \in \mathcal{D}^p$, we can get rewards from both actions in $\mathcal{A}$:
\begin{equation}
    \begin{aligned}
	\{\langle x,y_w,y_l, I\rangle\} \rightarrow \{
	&\langle \underbrace{(x,y_w, y_l)}_{\text{state}}, \underbrace{(y_w \succ y_l)}_{\text{action}}, \underbrace{(I)}_{\text{reward}}\rangle,\\
	&\langle \underbrace{(x,y_w, y_l)}_{\text{state}}, \underbrace{(y_l \succ y_w)}_{\text{action}}, \underbrace{(1-I)}_{\text{reward}}\rangle
	\}.
    \end{aligned}
\end{equation}
Because both actions appear at the same time,
we can define the policy generating $\mathcal{D}^p$ as:
\begin{equation}\label{pref policy}
	\Bar{\pi}^p(\mathfrak{a}|x,y_w,y_l)=1/2,\forall \mathfrak{a}\in\mathcal{A}.
\end{equation}
Define the preference-generating policy to be optimized as $\pi^p_\theta$, so the expected reward of $\pi^p_\theta$ is:
\begin{equation}\label{rl p}
	R(\pi^p_\theta)=\underset{\substack{x\sim \rho \\ \langle y_w, y_l \rangle \sim \mu(\cdot|x)
	\\ \mathfrak{a}\sim \pi^p_\theta(\cdot|x)}}{\mathbb{E}}[r^p(\mathfrak{a}|x)],
\end{equation}
and it's easy to see that $R(\Bar{\pi}^p)=1/2$.

We express the preference maximization objective in the reinforcement learning setting:
\begin{equation}\label{max obj}
	\max_{\pi^p_\theta}  R(\pi^p_\theta).
\end{equation}
Typically, the gradient of the objective \ref{max obj} needs to be estimated from samples continually collected by $\pi^p_\theta$,
which is data-inefficient. However, for preference maximization, we can directly estimate gradient from dataset $\mathcal{D}^p$,
which is off-policy.

\begin{theorem}\label{theorem pm}
Gradient of preference(reward) maximization objective \ref{max obj}
can be estimated from $\mathcal{D}^p$
\begin{align*}
\nabla_\theta R(\pi^p_\theta)&=\underset{\langle x,y_w,y_l,I\sim\mathcal{D} \rangle}{\mathbb{E}}\Bigl[\\
    I&\nabla_\theta\pi^p_\theta(y_w|x)\\
    +(1-I)&\nabla_\theta\pi^p_\theta(y_l|x)
\Bigr].
\end{align*}
\end{theorem}

We give proof for \ref{theorem pm} in Appendix. \ref{a:proof_pm}.

\subsection{Off-policy Preference Learning under KL-regularation}
Using the definition \ref{pair:policy},
we can reformulate reward maximization objective \ref{max obj} as:
\begin{equation}
	\max_{\pi_\theta} R(\pi^p_\theta),
\end{equation}
which means we optimize $\pi_\theta$ to maximize the preference reward for corresponding policy $\pi^p_\theta$.

Like RLHF's 'PPO-ptx' objective \ref{ppo-ptx},
we also add 'ptx' to KL-regularized preference maximization objective \ref{max obj} and get
\begin{equation}
	\begin{gathered}
	\max_{\pi_\theta} R(\pi^p_\theta) - \beta \mathbb{D}_{\text{KL}}[\pi_\theta||\pi_{\text{ref}}]
	+\gamma\underset{x\sim \mathcal{D}_\text{pretrain}}{\mathbb{E}}[\log\pi_\theta(x)]
	\end{gathered}.
\end{equation}
From theorem \ref{theorem pm}, preference maximization term $R(\pi^p_\theta)$ can be directly solved with 
off-policy policy gradient method. 
Pretraining data regularization term $\mathbb{E}_{x\sim\mathcal{D}_\text{pretrain}}\log \pi_\theta(x)$ can also be solved with offline data.
But the KL-regularization term $\mathbb{D}_\text{KL}[\pi_\theta||\pi_\text{ref}]$ needs to collect samples 
from $\pi_\theta(\cdot|x)$, which is on-policy.

\paragraph{Off-policy KL-regularation on reference policy $\pi_\text{ref}$}
Minimize $\mathbb{D}_{\text{KL}}[\pi_\theta||\pi_{\text{ref}}]$ needs on-policy
samples collection, which is data-inefficient. To solve the problem, we replace
$\mathbb{D}_\text{KL}[\pi_\theta||\pi_\text{ref}]$ with $-\mathbb{E}_{\langle x,y \rangle \sim \mathcal{D}_\text{ref}}[\log \pi_\theta(y|x)]$.
Like pretraining data regularization, $\mathbb{E}_{\langle x,y \rangle \sim \mathcal{D}_\text{ref}}[\log \pi_\theta(y|x)]$ can be computed with offline data.



\paragraph{Maximum Preference Optimization (MPO) loss} Using the modified regularization on $\pi_\text{ref}$,
we get the final objective of MPO:
\begin{equation}\label{MPO obj}
	\max_{\pi_\theta} R(\pi^p_\theta) + \beta\underset{\langle x, y \rangle\sim \mathcal{D}_\text{ref}}{\mathbb{E}}[\log \pi_\theta(y|x)]
	+\gamma\underset{x\sim\mathcal{D}_\text{pretrain}}{\mathbb{E}}[\log\pi_\theta(x)]
\end{equation}
Using objective \ref{MPO obj}, we define empirical MPO loss on dataset $\mathcal{D}^p$,$\mathcal{D}_{\text{ref}}$
and $\mathcal{D}_{\text{pretrain}}$:
\begin{equation}\label{MPO loss}
    \begin{aligned}
	\mathcal{L}_{\text{MPO}}&=\underbrace{\underset{
		\substack{
		\langle x,y_w,y_l, I\rangle \sim \mathcal{D}^p \\
		}
	}{-\mathbb{E}}[I\pi^p_\theta(y_w|x) + (1-I)\pi^p_\theta(y_l|x)]}_{
        \mathcal{L}_{\text{MPO\_RM}}
    }\\
	&\underbrace{-\beta\underset{\langle x,y\rangle\sim \mathcal{D}_\text{ref}}{\mathbb{E}}[\log\pi_\theta(y|x)]}_{ \mathcal{L}_\text{MPO\_REF} }\\
	&\underbrace{-\gamma
		\underset{x\sim \mathcal{D}_{\text{pretrain}}}
		{\mathbb{E}}[\log{\pi_\theta(x)}]}_{ \mathcal{L}_\text{MPO\_PRETRAIN} }.
    \end{aligned}
\end{equation}
Loss \ref{MPO loss} is very simple and intuitive, where $I\pi^p_\theta(y_w|x)+(1-I)\pi^p_\theta(y_l|x)$
tries to maximize preferences,
$\beta$ controls the strength of the regularization of SFT dataset,
and $\gamma$ controls the strength of the regularization of pretraining dataset.

\paragraph{Eliminate both the need for reward model and reference policy}
By using preference as reward, we eliminate the need for a reward model to approximate preference probability.
By replacing KL-regularization $\mathbb{D}_\text{KL}[\pi_\theta||\pi_\text{ref}]$ with offline dataset regularization $-\mathbb{E}_{\langle x,y \rangle \sim \mathcal{D}_\text{ref}}[\log\pi_\theta(y|x)]$,
we remove the need for the reference policy $\pi_\text{ref}$.
As a result, the MPO algorithm simplifies the learning process and reduces memory usage.

\subsection{Accelerated Training of MPO}
In practice, most preferences are deterministic, and we found DPO has a faster convergence rate than MPO.
This is because DPO weights the gradient by how incorrectly the implicit reward model orders the completions\cite{rafailov2023direct}.

Let's compare the gradient of DPO and MPO's losses.
    \begin{align*}
        \nabla_\theta\mathcal{L}_\text{DPO}(\pi_\theta;\pi_\text{ref})&=
        -\beta\underset{\langle x,y_w,y_l,I \rangle \sim \mathcal{D}^p}{\mathbb{E}}\Bigl[\\
        &I{\color{red}\sigma(r_\theta(x,y_l) - r_\theta(x,y_w))}\cdot\\
        &\nabla_\theta(\log\pi_\theta(y_w|x)-\log \pi_\theta(y_l|x))\\
        +&(1-I){\color{red}\sigma(r_\theta(x,y_w) - r_\theta(x,y_l))}\cdot\\
        &\nabla_\theta(\log\pi_\theta(y_l|x)-\log \pi_\theta(y_w|x))
        \Bigr]
    \end{align*}


Define the reward maximization part of MPO loss as 
\begin{align*}
\mathcal{L}_\text{MPO\_RM}(\pi_\theta)=&-\underset{\langle x,y_w,y_l,I\rangle\sim\mathcal{D}^p}{\mathbb{E}}
\Bigl[\\
&I\pi^p_\theta(y_w|x) + (1-I)\pi^p_\theta(y_l|x)
\Bigr]
\end{align*}
Then the gradient of the $\mathcal{L}_\text{MPO\_RM}$ is

    \begin{align*}
        \nabla_\theta\mathcal{L}_\text{MPO\_RM}&=-\underset{\langle x,y_w,y_l,I\rangle \sim\mathcal{D}}{\mathbb{E}}\Bigl[\\
    I&{\color{red}\pi^p_\theta(y_w|x)\pi^p_\theta(y_l|x)}\cdot\\
    &\nabla_\theta(\log \pi_\theta(y_w|x)-\log \pi_\theta(y_l|x))\\
    +(1-I)&{\color{red}\pi^p_\theta(y_l|x)\pi^p_\theta(y_w|x)}\cdot\\
    &\nabla_\theta(\log \pi_\theta(y_l|x)-\log \pi_\theta(y_w|x))
    \Bigr],
    \end{align*}
which is near zero when $\pi^p_\theta(y_w\succ y_l|x)$ or $\pi^p_\theta(y_l\succ y_w|x)$ nears zero, and this will slow down the 
learning process.
\paragraph{Preference Matching}
Like the gradient of DPO, we can weight the gradient by how incorrectly the model orders the completions.
Define the weighted  gradient as $\nabla_\theta\mathcal{L}^w_\text{MPO\_RM}$:
    \begin{align*}
        \nabla_\theta\mathcal{L}^w_\text{MPO\_RM}&=-\underset{\langle x,y_w,y_l,I\rangle \sim\mathcal{D}^p}{\mathbb{E}}\Bigl[\\
    I&{\color{red}\pi^p_\theta(y_l|x)}\cdot\\
    &\nabla_\theta(\log \pi_\theta(y_w|x)-\log \pi_\theta(y_l|x))\\
    +(1-I)&{\color{red}\pi^p_\theta(y_w|x)}\cdot\\
    &\nabla_\theta(\log \pi_\theta(y_l|x)-\log \pi_\theta(y_w|x))
    \Bigr].
    \end{align*}
The loss $\mathcal{L}^w_{\text{MPO\_RM}}$ corresponding to $\nabla_\theta\mathcal{L}^w_{\text{MPO\_RM}}$ is
\begin{align*}
        \mathcal{L}^w_\text{MPO\_RM}=&-\underset{\langle x,y_w,y_l,I\rangle \sim\mathcal{D}^p}{\mathbb{E}}\Bigl[\\
    &I\log \pi^p_\theta(y_w|x)
    +(1-I)\log \pi^p_\theta(y_l|x) 
    \Bigr],
\end{align*}
which is the cross entropy loss between $\pi^{*^p}(\cdot|x,y_w,y_l)$ and $\pi^p_\theta(\cdot|x,y_w,y_l)$,
where $\pi^{*^p}(\cdot|x,y_w,y_l)$ is the local distribution introduced by $p^*(y_w\succ y_l|x)$ and $p^*(y_l\succ y_w| x)$.
So the weighted loss learns to match the preference probability.

\paragraph{Preference Maximization}
When preference is not deterministic, the weighted loss $\mathcal{L}^w_\text{MPO\_RM}$ will be suboptimal.
We can switch to loss $\mathcal{L}_\text{MPO\_RM}$ after training enough iterations with $\mathcal{L}^w_\text{MPO\_RM}$.


Combine the weighted gradient $\nabla_\theta\mathcal{L}^w_\text{MPO\_RM}$, we summarize the MPO in Algorithm\ref{alg:mpo}.

\begin{algorithm}[tb]
   \caption{Maximum Preference Optimization (MPO)}
   \label{alg:mpo}
\begin{algorithmic}
   \STATE {\bfseries Input:} {\raggedright Preference dataset $\mathcal{D}^p$, reference dataset $\mathcal{D}_\text{ref}$, \\\quad pretrain dataset $\mathcal{D}_\text{pretrain}$}, config $weighted\_iter\_num$
   \STATE $iter \leftarrow 0$
   \REPEAT
    \IF{$iter$ $>$ $weighted\_iter\_num$}
    \STATE Sample batch from $\mathcal{D}^p$, compute gradient $g_\text{rm}$ \\ \quad using $\nabla_\theta\mathcal{L}^w_\text{MPO\_RM}$
    \ELSE
    \STATE Sample batch from $\mathcal{D}^p$, compute gradient $g_\text{rm}$ \\ \quad using $\nabla_\theta\mathcal{L}_\text{MPO\_RM}$
    \ENDIF
   \STATE Sample batch from $\mathcal{D_\text{ref}}$, compute gradient $g_\text{ref}$ \\ \quad using $\nabla_\theta\mathcal{L}_\text{MPO\_REF}$
   \STATE Sample batch from $\mathcal{D_\text{ref}}$, compute gradient $g_\text{pretrain}$ \\ \quad using $\nabla_\theta\mathcal{L}_\text{MPO\_PRETRAIN}$
   \STATE let $g_\theta=g_\text{rm}+g_\text{ref}+g_\text{pretrain}$ 
   \STATE update $\pi^s_\theta$ with $g_\theta$
   \STATE $iter \leftarrow iter + 1$
   \UNTIL{converged.}
\end{algorithmic}
\end{algorithm}






\section{Experiments}


This section presents a sequence of empirical studies assessing the performance of the MPO algorithm in learning preferences without a reference policy and its capability to prevent overfitting via off-policy KL regularization. The experimental design is aligned with DPO and involves two principal phases: SFT and preference alignment. During SFT, the base model, Mistral-7B-v0.1\footnote{\url{https://huggingface.co/mistralai/Mistral-7B-v0.1}}, is fine-tuned using datasets consisting of point-wise prompt-response pairs. In the subsequent preference alignment phase, we refine the model's text generation policy by employing datasets that pair prompts with human preference judgments. The hyperparameters remained constant across all preference learning experiments, including a learning rate of 5e-7, a batch size of 32, and a training duration of one epoch.

\subsection{Preference Learning without Reference Policy}

Typically, DPO and IPO algorithms rely on a reference policy to guide regularized preference learning. In contrast, the MPO algorithm directly engages in off-policy preference maximization without the need for such a reference model. We probe the preference learning prowess of MPO using HH-RLHF\footnote{\url{https://huggingface.co/datasets/Anthropic/hh-rlhf}} dataset which comprises paired responses for each prompt—labeled as ’chosen’ or ’rejected’. During the SFT phase, the ’chosen’ responses were treated as the target for the given prompts. In the preference alignment phase, we refined the policy model with paired responses. As delineated in Table 1, the empirical results demonstrate accuracy rates across 14 benchmarks. When benchmarked against the performance of the SFT model, MPO exhibited a mean accuracy enhancement of 3.7, closely matching the performance of DPO. These outcomes indicate that MPO is capable of learning human preferences as effectively as DPO, even in the absence of a reference policy.

\begin{table}[t]
    \caption{Comparison of preference learning ability of DPO, IPO and MPO (without reference model) on 14 benchmarks.}
    \label{preference-table}
    \vskip 0.15in
    \begin{center}
    \begin{small}
    \begin{sc}
    \begin{tabular}{lcccr}
    \toprule
    Task & SFT & DPO & IPO & MPO \\
    \midrule
    ANLI-R1                 & 0.424  & 0.476&0.443&0.464 \\
    ANLI-R2                 & 0.391  & 0.445&0.417&0.443 \\
    ANLI-R3                 & 0.416  & 0.444&0.431&0.443 \\
    ARC-C                   & 0.564  & 0.605&0.565&0.620 \\
    ARC-E                & 0.838  & 0.846&0.838&0.849 \\
    BoolQ                    & 0.830  & 0.844&0.834&0.853 \\
    HellaSwag                & 0.773  & 0.820&0.773&0.810 \\
    OpenBookQA               & 0.472  & 0.502&0.472&0.502 \\
    PiQA                     & 0.818  & 0.825&0.816&0.822 \\
    RTE                      & 0.661  & 0.718&0.693&0.737 \\
    Toxigen                  & 0.573  & 0.660&0.572&0.656 \\
    TruthfulQA1              & 0.272  & 0.340&0.285&0.350 \\
    TruthfulQA2              & 0.399  & 0.488&0.422&0.490 \\
    WiC                      & 0.577  & 0.571&0.558&0.528 \\
    Winogrande               & 0.739  & 0.721&0.740&0.733 \\
    \hline
    Average                  & 0.583  & \textbf{0.620}&0.591 &\textbf{0.620} \\
    
    \bottomrule
    \end{tabular}
    \end{sc}
    \end{small}
    \end{center}
    \vskip -0.1in
    \end{table}

\subsection{Off-policy KL-regularization}

Due to the failure of KL regularization, both DPO and IPO algorithms can enhance the performance on downstream related tasks based on preference data, but they may decrease the performance of tasks in the SFT or pretrain stage that have a lower correlation with preference data.

To explore the effectiveness of off-policy KL-regularization in mitigating overfitting, we designed a series of experiments targeting the HellaSwag\footnote{\url{https://rowanzellers.com/hellaswag/}}, GSM8K\footnote{\url{https://huggingface.co/datasets/gsm8k}}, and MATH\footnote{\url{https://huggingface.co/datasets/hendrycks/competition_math}} benchmarks. We utilize the HellaSwag and MetaMathQA\footnote{\url{https://huggingface.co/datasets/meta-math/MetaMathQA}} datasets for SFT and employ the HH-RLHF dataset for preference learning. Our results, detailed in Tables 2 and 3, show that DPO and IPO exhibit a significant decline in performance on the HellaSwag, GSM8K, and MATH benchmarks after preference learning since these three tasks are irrelevant to HH-RLHF. In contrast, MPO maintains its performance. The incorporation of reference regularization into preference loss successfully mitigates the overfitting tendencies evident in DPO and IPO. Furthermore, our scores on the TruthfulQA\footnote{\url{https://huggingface.co/datasets/truthful_qa}} benchmark, which is relevant to the HH-RLHF task, confirm that off-policy constraints do not adversely affect preference learning. These results underscore the resilience of MPO in aligning with human preferences and its superior capability in curbing overfitting compared to DPO and IPO. The research opens avenues for the development of more durable preference learning strategies independent of reference policies.

\begin{table}[t]
    \caption{Comparison of Regularization ability of DPO, IPO and MPO on HellaSwag}.
    \label{regularization-hellaswag-table}
    \vskip 0.15in
    \begin{center}
    \begin{small}
    \begin{sc}
    \begin{tabular}{lcccr}
    \toprule
    Task & SFT & DPO & IPO & MPO \\
    \midrule
    HellaSwag                & 0.841 &0.801&0.817&\textbf{0.861}\\
    TruthfulQA1              & 0.280 &\textbf{0.353}&0.313&0.340\\
    TruthfulQA2              & 0.435 &\textbf{0.498}&0.454&0.484\\
    \hline
    Average                  & 0.519 &0.551&0.528&\textbf{0.562}\\
    
    \bottomrule
    \end{tabular}
    \end{sc}
    \end{small}
    \end{center}
    \vskip -0.1in
    \end{table}
    
    \begin{table}[t]
    \caption{Comparison of Regularization ability of DPO, IPO and MPO on GSM8K and MATH}.
    \label{regularization-metamath-table}
    \vskip 0.15in
    \begin{center}
    \begin{small}
    \begin{sc}
    \begin{tabular}{lcccr}
    \toprule
    Task & SFT & DPO & IPO & MPO \\
    \midrule
    GSM8K             & 0.648&0.589&0.585&\textbf{0.671}\\
    MATH              & \textbf{0.170}&0.129&0.122&0.160\\
    TruthfulQA1       & 0.312&\textbf{0.351}&0.327&0.342\\
    TruthfulQA2       & 0.452&\textbf{0.491}&0.462&0.485\\
    \hline
    Average           & 0.396&0.390&0.374&\textbf{0.415}\\
    
    \bottomrule
    \end{tabular}
    \end{sc}
    \end{small}
    \end{center}
    \vskip -0.1in
    \end{table}

\section{Conclusion and Future Works}
We have introduced MPO, an off-policy policy based algorithm for preference learning.
Unlike most off-policy algorithm which are derived using KL-regularized reward maximization,
MPO is directly derived from an importance sampling view. Because MPO maximizes reward without 
the need of reference policy, MPO simplifies the learning process and saves memory usage.
MPO uses a sample-based forward KL regularization term to prevent overfitting preference data,
which makes KL-regularization truly effective and data efficient.

We haven't tested the ratio of reference or pretrain data and different regularization weights 
to balance preference learning and reference regularization. When use reference data as regularization,
this will retrain the reference data, which will lead to overfitting of the reference data.
How to balance preference and regularization weight to get the best performance, and how to avoid reference
data overfitting are left for future work.

\bibliographystyle{unsrtnat}
\bibliography{main}  






\newpage
\appendix
\onecolumn
\section{Mathematical Derivations}

\subsection{Proof of theorem \ref{theorem pm}}\label{a:proof_pm}
\begin{theorem}
Gradient of preference(reward) maximization objective \ref{max obj}
can be estimated from $\mathcal{D}^p$
\begin{align*}
\nabla_\theta R(\pi^p_\theta)&=\underset{\langle x,y_w,y_l,I \rangle\sim\mathcal{D}^p}{\mathbb{E}}\Bigl[
    I\nabla_\theta\pi^p_\theta(y_w|x)
    +(1-I)\nabla_\theta\pi^p_\theta(y_l|x)
\Bigr].
\end{align*}
\end{theorem}

\begin{proof}
According to REINFORCE algorithm, policy gradient of the \ref{max obj} is:
\begin{equation}\label{pg gradient}
	\nabla_\theta R(\pi^p_\theta)= 
	\underset{\substack{x\sim\rho \\ \langle y_w,y_l\rangle \sim \mu(\cdot|x) \\ \mathfrak{a}\sim \pi^p_\theta(\cdot|x)}}
	{\mathbb{E}}[r^p(\mathfrak{a}|x)\nabla_\theta \log \pi^p_\theta(\mathfrak{a}|x)].
\end{equation}

Using importance sampling, gradient \ref{pg gradient} be expressed as:
\begin{equation}\label{off gradient}
    \begin{aligned}
	\nabla_\theta R(\pi^p_\theta)=&\underset{\substack{x\sim\rho \\ \langle y_w,y_l\rangle \sim \mu(\cdot|x) \\ \mathfrak{a}\sim \Bar{\pi}^p(\cdot|x)}}{\mathbb{E}}\Bigl[\frac{
		\pi^p_\theta(\mathfrak{a}|x)
	}{\Bar{\pi}^p(\mathfrak{a}|x)}r^p(\mathfrak{a}|x)\nabla_\theta\log \pi^p_\theta(\mathfrak{a}|x)\Bigr]\\
    =&\underset{\substack{x\sim\rho \\ \langle y_w,y_l\rangle \sim \mu(\cdot|x)}}{\mathbb{E}}\Bigl[
		r^p(y_w|x)\pi^p_\theta(y_w|x)\nabla_\theta\log \pi^p_\theta(y_w|x)
		+r^p(y_l|x)\pi^p_\theta(y_l|x)\nabla_\theta\log \pi^p_\theta(y_l|x)
    \Bigr]\\
    =&\underset{\substack{x\sim\rho \\ \langle y_w,y_l\rangle \sim \mu(\cdot|x)}}{\mathbb{E}}\Bigl[
		r^p(y_w|x)\nabla_\theta \pi^p_\theta(y_w|x)
		+r^p(y_l|x)\nabla_\theta \pi^p_\theta(y_l|x)
    \Bigr]\\
    =&\underset{\langle x,y_w,y_l,I \rangle\sim \mathcal{D}^p}{\mathbb{E}}\Bigl[
        I\nabla_\theta\pi^p_\theta(y_w|x) + (1-I)\nabla_\theta\pi^p_\theta(y_l|x)
    \Bigr].
    \end{aligned}
\end{equation}
\end{proof}

\end{document}